\documentclass{article}

\usepackage{icml2009}

\usepackage{graphicx} 
\usepackage{subfigure}
\usepackage{alex}
\usepackage{mlapa}
\usepackage{mdwlist}
\usepackage{color}
\def\N{m}
\usepackage{times}

\icmltitlerunning{Feature Hashing for Large Scale Multitask Learning}

\begin{document}

\twocolumn[
\icmltitle{Feature Hashing for Large Scale Multitask Learning}

\icmlauthor{}{}
%
\icmlauthor{Kilian Weinberger}{kilian@yahoo-inc.com}
\icmlauthor{Anirban Dasgupta}{anirban@yahoo-inc.com}
\icmlauthor{Josh Attenberg}{josh@cis.poly.edu}
\icmlauthor{John Langford}{jl@hunch.net}
\icmlauthor{Alex Smola}{alex@smola.org}
\icmladdress{Yahoo! Research, 2821 Mission College Blvd.,
   Santa Clara, CA 95051 USA}
%
\icmladdress{{\bf Keywords}: kernels, concentration inequalities, document classification, classifier personalization, multitask learning}
\vskip 0.3in
]

\begin{abstract}
Empirical evidence suggests that hashing is an effective strategy for dimensionality reduction and practical nonparametric estimation. In this paper we provide exponential tail bounds for feature hashing and show that the interaction between random subspaces is negligible with high probability. We demonstrate the feasibility of this approach with experimental results for a new use case --- multitask learning with hundreds of thousands of tasks.
\end{abstract}

\section{Introduction}
\label{sec:intro}

Kernel methods use inner products as the basic tool for comparisons between objects. That is, given objects $x_1,\dots, x_n \in \Xcal$ for some domain $\Xcal$, they rely on
\begin{align}
  \label{eq:kernel}
  k(x_i,x_j) := \inner{\phi(x_i)}{\phi(x_j)}
\end{align}
to compare the features $\phi(x_i)$ of $x_i$ and $\phi(x_j)$ of $x_j$ respectively. 

Eq.~(\ref{eq:kernel}) is often famously referred to as the \emph{kernel-trick}. It allows the use of inner products between very high
dimensional feature vectors $\phi(x_i)$ and $\phi(x_j)$ \emph{implicitly} through the definition of a positive semi-definite kernel matrix $k$
without ever having to compute a vector $\phi(x_i)$ directly. This can be particularly powerful in classification settings where the original
input representation has a non-linear decision boundary. Often, linear separability can be achieved in a  high dimensional feature space $\phi(x_i)$.


In practice, for example in text classification, researchers frequently encounter the opposite problem:  the original input space is almost linearly separable (often because of the existence of handcrafted non-linear features), yet, the training set may be prohibitively large in size and very high dimensional. In such a case, there is no need to map the input vectors into a higher dimensional feature space. Instead, limited memory makes storing a kernel matrix infeasible.

For this common scenario several authors have recently proposed an alternative, but highly complimentary variation of the kernel-trick, which  we refer to as the \emph{hashing-trick}: one \emph{hashes} the high dimensional input vectors $x$ into a \emph{lower} dimensional feature space $\RR^m$ with \mbox{$\phi:\Xcal\rightarrow\RR^m$}~\cite{LaLiSt07,Shietal09}. The parameter vector of a classifier can therefore live in $\RR^m$ instead of in $\RR^n$ with kernel matrices or $\RR^d$ in the original input space, where $m\ll n$ and $m\ll d$. Different from random projections, the hashing-trick preserves sparsity and introduces no additional overhead to store projection matrices.



To our knowledge, we are the first to provide exponential tail bounds on the canonical distortion of these hashed inner products. We also show that the hashing-trick can be particularly powerful in multi-task learning scenarios where the original feature spaces are the cross-product of the data, $\Xcal$, and the set of tasks, $U$. We show that one can use different hash functions for each task $\phi_1,\dots,\phi_{|U|}$ to map the data into one joint space with little interference.

While many potential applications exist for the hashing-trick, as a particular case study we focus on collaborative email spam filtering. In this scenario, hundreds of thousands of users collectively label emails as \emph{spam} or \emph{not-spam}, and each user expects a personalized classifier that reflects their particular preferences. Here, the set of tasks, $U$, is the number of email users (this can be very large for open systems such as \emph{Yahoo Mail}\texttrademark or \emph{Gmail}\texttrademark), and the feature space spans the union of vocabularies in multitudes of languages.

This paper makes four main contributions: 1.~In section~\ref{sec:hashkernels} we introduce specialized hash functions with unbiased inner-products that are directly applicable to a large variety of kernel-methods. 2.~In section~\ref{sec:analysis} we provide exponential tail bounds that help explain why hashed feature vectors have repeatedly lead to, at times surprisingly, strong empirical results. 3.~Also in section~\ref{sec:analysis} we show that the interference between independently hashed subspaces is negligible with high probability, which allows large-scale multi-task learning in a very compressed space. 4.~In section~\ref{sec:results}  we introduce collaborative email-spam filtering as a novel application for hash representations and provide experimental results on large-scale real-world spam data sets.


\section{Hash Functions}
\label{sec:hashkernels}

We introduce a variant on the hash kernel proposed by
\cite{Shietal09}. This scheme is modified through the introduction of a \emph{signed}
sum of hashed features whereas the original hash kernels use an
\emph{unsigned} sum. This modification leads to an unbiased estimate, which we demonstrate and further utilize in the following section.

\begin{definition}
Denote by $h$ a hash function $h:\NN\rightarrow \cbr{1,\dots,\N}$. Moreover, denote by $\xi$ a hash function $\xi:\NN\rightarrow\cbr{\pm 1}$. Then for vectors ${x},{x}'\in \ell_2$ we define the hashed feature map $\phi$ and the corresponding inner product as
  \begin{align}
    \phi^{(h,\xi)}_i(x) & = \hspace{-2mm} \sum_{j: h(j) = i} \hspace{-2mm}
    \xi(i) x_i \\
    \text{ and }
    \inner{x}{x'}_\phi & := \inner{\phi^{(h,\xi)}(x)}{\phi^{(h,\xi)}(x')}.
  \end{align}\label{def:hash}
\end{definition}
\vspace{-3mm}
%
Although the hash functions in definition~\ref{def:hash} are defined over the natural numbers $\NN$, in practice we often consider hash functions over arbitrary strings. These are equivalent, since each finite-length string can be represented by a unique natural number.

Usually, we abbreviate the notation $\phi^{(h,\xi)}(\cdot)$ by just
$\phi(\cdot)$. Two hash functions $\phi$ and $\phi'$ are different
when $\phi = \phi^{(h,\xi)}$ and $\phi'= \phi^{(h',\xi')}$ such that
either $h'\neq h$ or $\xi \neq \xi'$.  The purpose of the binary hash
$\xi$ is to remove the bias inherent in the hash kernel of
\cite{Shietal09}.

In a multi-task setting, we obtain instances in combination with tasks, $(x,u)\in \Xcal\times U$.
We can naturally extend our definition~\ref{def:hash} to hash pairs, and will write $\phi_u(x)=\phi(x,u)$.

\section{Analysis}
\label{sec:analysis}
The following section is dedicated to theoretical analysis of hash
kernels and their applications. In this sense, the present paper continues
 where~\cite{Shietal09} falls short: we prove exponential tail
bounds. These bounds hold for general hash kernels, which we later
apply to show how hashing enables us to do large-scale multitask learning efficiently. We start with a simple lemma about the bias and variance of the hash kernel. The proof of this lemma appears in
appendix~\ref{proof:mean}.

\begin{lemma}
  \label{th:meanvar}
The hash kernel is unbiased, that is $\Eb_\phi[\inner{x}{x'}_\phi] =
\inner{x}{x'}$. Moreover, the variance is
$\sigma^2_{x,x'} = \frac{1}{\N}\left(\sum_{i\neq j} x_i^2 {x_j'}^2 + x_i {x}_i' x_j {x}_j'  \right)$,
and thus, for $\|x\|_2 = \|x'\|_2 = 1$, $\sigma_{x,x'}^2 = O\left(\frac{1}{\N}\right)$.
\end{lemma}
%

This suggests that typical values of the hash kernel should be concentrated within $O(\frac{1}{\sqrt{\N}})$ of the target value. We use Chebyshev's inequality to show that half of all observations are within a range of $\sqrt{2}\sigma$. This, together with an indirect application of Talagrand's convex distance inequality via the result of \cite{liberty2008dfr}, enables us to construct exponential tail bounds.



\subsection{Concentration of Measure Bounds}
In this subsection we show that under a hashed feature-map the length
of each vector is preserved with high probability.
Talagrand's inequality~\cite{Ledoux01} is a key tool for the proof of
the following theorem (detailed in the appendix~\ref{sec:proofs}).
\begin{theorem}
  \label{th:key}
  Let $\epsilon <1$ be a fixed constant and $x$  be a given instance such that $\|x\|_2 = 1$. If $m \geq 72\log(1/\delta)/\epsilon^2$ and $\|x\|_\infty  \leq \frac{\epsilon}{18\sqrt{\log(1/\delta)\log(m/\delta)}}$, we have that
  \begin{align}
\Pr [ | \|x\|^2_\phi - 1 | \geq \epsilon ] \leq 2\delta.
\end{align}
\end{theorem}
Note that an analogous result would also hold for the original hash kernel of \cite{Shietal09}, the only modification being the associated bias terms. The above result can also be utilized to show a concentration bound on the inner product between two general vectors $x$ and $x'$.

\begin{corollary}
\label{cor:inner}
For two vectors $x$ and $x'$, let us define
\begin{align*}
\sigma&:=\max(\sigma_{x,x}, \sigma_{x',x'}, \sigma_{x-x',x-x'})\\
\eta&:= \max\left(\frac{\|x\|_\infty}{\|x\|_2},\frac{\|x'\|_\infty}{\|x'\|_2},\frac{\|x-x'\|_\infty}{\|x-x'\|_2} \right).
\end{align*}
 Also let $\Delta = \nbr{x}^2 + \nbr{x'}^2 + \nbr{x - x'}^2$. If  $m \geq \Omega(\frac{1}{\epsilon^2}\log(1/\delta))$ and $\eta = O(\frac{\epsilon}{\log(m/\delta)})$, then we have that
\begin{align*}
&\Pr\sbr{|\inner{x}{x'}_{\phi}\!-\! \inner{x}{x'}|\! >\! \epsilon \Delta/2}\!\nonumber \!<\! \delta.
\end{align*}
\end{corollary}
The proof for this corollary can be found in appendix~\ref{sec:innerproof}. We can also extend the bound in Theorem~\ref{th:key} for the maximal canonical distortion over large sets of distances between
vectors as follows:
\begin{corollary}
  \label{cor:union}
  If $m \geq \Omega(\frac{1}{\epsilon^2}\log(n/\delta))$ and $\eta = O(\frac{\epsilon}{\log(m/\delta)})$. Denote by $X = \cbr{x_1, \ldots, x_n}$ a set of vectors which
  satisfy $\nbr{x_i - x_j}_\infty \leq \eta \nbr{x_i - x_j}_2$ for all
  pairs $i, j$. In this case with probability $1 - \delta$ we have for
  all $i,j$
  \begin{align*}
    \frac{|\nbr{x_i-x_j}_\phi^2 - \nbr{x_i-x_j}_2^2|}{\nbr{x_i-x_j}_2^2}\leq \epsilon.
  \end{align*}
\end{corollary}
This means that the number of observations $n$ (or correspondingly the size of the un-hashed kernel matrix) only enters \emph{logarithmically} in the
analysis.
\begin{proof}
  We apply the bound of Theorem~\ref{th:key} to each distance
  individually. Note that each vector $x_i - x_j$ satisfies the conditions of the theorem, and hence for each vector $x_i - x_j$, we preserve the distance upto a factor of $(1 \pm \epsilon)$ with probability $1- \frac{\delta}{n^2}$. Taking the union bound over all pairs gives us the result.
\end{proof}
\vspace{-3mm}

\subsection{Multiple Hashing}
\label{subsec:multiplehashing}
Note that the tightness of the union bound in Corollary~\ref{cor:union} depends crucially on the
magnitude of $\eta$. In other words, for large values of $\eta$, that
is, whenever some terms in $x$ are very large, even a single collision
can already lead to significant distortions of the embedding. This
issue can be amended by trading off sparsity with variance. A vector of unit
length may be written as $(1, 0, 0, 0, \ldots)$, or as $\left(\frac{1}{\sqrt{2}},
\frac{1}{\sqrt{2}}, 0, \ldots\right)$, or more generally as a vector with $c$
nonzero terms of magnitude $c^{-\frac{1}{2}}$. This is relevant, for
instance whenever the magnitudes of $x$ follow a known pattern,
e.g.\ when representing documents as bags of words since we may simply
hash frequent words several times. The following corollary gives an
intuition as to how the confidence bounds scale in terms of the
replications:
\begin{lemma}
   \label{lem:rep}
  If we let  $x' = \frac{1}{\sqrt{c}}(x, \ldots, x)$ then:
\begin{enumerate}
  \item It is norm preserving: $\nbr{x}_2 = \nbr{x'}_2.$
  \item It reduces component magnitude by $\frac{1}{\sqrt{c}} = \frac{\nbr{x'}_\infty}{\nbr{x}_\infty}.$
  \item Variance increases to $\sigma^2_{x', x'}\! =\! \frac{1}{c}\sigma^2_{x,x}\!+\! \frac{c-1}{c}2\nbr{x}_2^4.$
\end{enumerate}
\end{lemma}
Applying Lemma~\ref{lem:rep} to Theorem~\ref{th:key}, a large magnitude can be decreased at
the cost of an increased variance.

\subsection{Approximate Orthogonality}

For multitask learning, we must learn a different parameter vector for
each related task. When mapped into the same hash-feature space
we want to ensure that there is little interaction between the
different parameter vectors. Let $U$ be a set of different tasks,
$u\in U$ being a specific one.  Let $w$ be a combination of the
parameter vectors of tasks in $U\setminus \{u\}$. We show that for any
observation $x$ for task $u$, the interaction of $w$ with $x$ in the
hashed feature space is minimal.  For each $x$, let the image of $x$
under the hash feature-map for task $u$ be denoted as $\phi_u(x)=\phi^{(\xi, h)}((x,u))$.
%

\begin{theorem}
  \label{th:ortho}
  Let $w \in \RR^\N$ be a parameter vector for tasks in $U\setminus \{u\}$. In this case the value of the inner product
  $\inner{w}{\phi_u(x)}$ is
  bounded by
  \begin{align}
    \nonumber
    \Pr\cbr{\abr{\inner{w}{\phi_u(x)}} > \epsilon}
     \leq 2 e^{
      -\frac{\epsilon^2/2}{\N^{-1} \nbr{w}_2^2 \nbr{x}_2^2 + \epsilon
        \nbr{w}_\infty \nbr{x}_\infty/3}}
  \end{align}
\end{theorem}
\begin{proof}
  We use Bernstein's inequality~\cite{Bernstein46}, which states that for independent random variables $X_j$, with $\Eb\sbr{X_j} = 0$, if $C>0$ is such that $|X_j|\leq C$, then
\begin{equation}
	\Pr\sbr{\sum_{j=1}^n\! X_j\!>\!t}\!\leq\! \exp{\!\left(\!-\frac{t^2/2}{\sum_{j=1}^n\mathbf{E}\sbr{X_j^2}+Ct/3}\!\right)}.	\label{def:Bernstein}
\end{equation}


We have to compute the concentration property of \\$\inner{w}{\phi_u(x)}=\sum_j x_j \xi(j) w_{h(j)}$.  Let $X_j = x_j \xi(j)w_{h(j)}$. By the definition of $h$ and $\xi$, $X_j$ are independent. Also, for each $j$, since $w$ depends only on the hash-functions for $U\setminus\{u\}$, $w_h(j)$ is independent of $\xi(j)$. Thus, $\Eb[X_j] = \Eb_{(\xi, h)}\sbr{x_j \xi(j) w_{h(j)}} = 0$. For each $j$, we also have $|X_j| < \nbr{x}_\infty
  \nbr{w}_\infty=:C$. Finally, $\sum_j \Eb[X_j^2]$ is given by
  \begin{align*}
    \Eb\sbr{\sum_j (x_j \xi(j) w_{h(j)})^2} = {\textstyle \frac{1}{\N}}
    \sum_{j,\ell} x_j^2 w_\ell^2 = {\textstyle \frac{1}{\N}} \nbr{x}_2^2 \nbr{w}_2^2
  \end{align*}
  The claim follows by plugging both terms and $C$ into the Bernstein inequality~(\ref{def:Bernstein}).
\end{proof}

Theorem~\ref{th:ortho} bounds the influence of unrelated tasks with any particular instance.
In section~\ref{sec:results} we demonstrate the real-world applicability with empirical results on a large-scale multi-task learning problem.

\section{Applications}

The advantage of feature hashing is that it allows for significant
storage compression for parameter vectors: storing
$w$ in the raw feature space naively requires $O(d)$ numbers, when $w \in
\RR^d$. By hashing, we are able to reduce this to $O(\N)$ numbers while
avoiding costly matrix-vector multiplications common in
Locally Sensitive Hashing. In addition, the sparsity of the resulting vector is preserved.

The benefits of the hashing-trick leads to applications in almost all areas of machine learning and beyond. In particular, feature hashing is extremely useful whenever large numbers of parameters with redundancies need to be stored within bounded memory capacity.

\paragraph{Personalization}
\label{par:personalization}
One powerful application of feature hashing is found in multitask learning. Theorem~\ref{th:ortho} allows us to hash multiple classifiers for different tasks into one feature space with little interaction.
%
To illustrate, we explore this setting in the context of spam-classifier personalization.

Suppose we have thousands of users $U$ and want to perform related but not identical classification tasks for each of the them. 
Users provide labeled data by marking emails as
\emph{spam} or \emph{not-spam}. Ideally, for each user $u\in U$, we
want to learn a predictor $w_u$ based on the data of that user
solely. However, webmail users are notoriously lazy in labeling emails
and even those that do not contribute to the training data
expect a working spam filter. Therefore, we also need to learn an
additional global predictor $w_0$ to allow data sharing amongst all
users.

Storing all predictors $w_i$ requires $O(d\times (|U|+1))$ memory. In a task like collaborative spam-filtering, $|U|$, the number of users can be in the hundreds of thousands and the size of the vocabulary is usually in the order of millions. The naive way of dealing with this is to eliminate all infrequent tokens. However, spammers target this memory-vulnerability by maliciously misspelling words and thereby creating highly infrequent but spam-typical tokens that ``fall under the radar'' of conventional classifiers. Instead, if all words are hashed into a finite-sized feature vector, infrequent but class-indicative tokens get a chance to contribute to the classification outcome. Further, large scale spam-filters (e.g. \emph{Yahoo Mail}\texttrademark or \emph{GMail}\texttrademark) typically have severe memory and time constraints, since they have to handle billions of emails per day. To guarantee a finite-size memory footprint we hash all weight vectors $w_0,\dots,w_{|U|}$ into a joint, significantly smaller, feature space $\RR^m$ with different hash functions $\phi_0,\dots,\phi_{|U|}$. The resulting hashed-weight vector $w_h\in\RR^m$ can then be written as:
\begin{equation}
	w_h = \phi_0 (w_0) + \sum_{u\in U} \phi_u(w_u).\label{eq:hybriddef}
\end{equation}
Note that in practice the weight vector $w_h$ can be learned directly in the hashed space. All un-hashed weight vectors never need to be computed. Given a new document/email $x$ of user $u\in U$, the prediction task now consists of calculating $\inner{\phi_0(x) + \phi_u(x)}{w_h}$. Due to hashing we have two sources of error -- distortion $\epsilon_d$ of the hashed inner products and the interference with other hashed weight vectors $\epsilon_i$. More precisely:
\begin{equation}
\inner{\phi_0(x)+\phi_u(x)}{w_h}=\inner{x}{w_0+w_u}+\epsilon_d+\epsilon_i.
\end{equation}



The interference error consists of all collisions between $\phi_0(x)$ or $\phi_u(x)$ with hash functions of other users,
\begin{equation}
\epsilon_i\!=\hspace{-2ex}\sum_{v\in U, v\neq 0}\hspace{-2ex}\inner{\phi_0(x)}{\phi_v(w_v)}+\hspace{-2ex}\sum_{v\in U, v\neq u}\hspace{-2ex}\inner{\phi_u(x)}{\phi_v(w_v)}.\label{def:epsiloni}
\end{equation}
To show that $\epsilon_i$ is small with high probability we can apply Theorem~\ref{th:ortho} twice, once for each term of (\ref{def:epsiloni}). We consider each user's classification to be a separate task, and since $\sum_{v\in U, v\neq 0}w_v$ is independent of the hash-function $\phi_0$, the conditions of Theorem~\ref{th:ortho} apply with $w = \sum_{v\neq 0} w_v$  and we can employ it to bound the second term, $\sum_{v\in U, v\neq 0}\inner{\phi_u(x)}{\phi_u(w_v)}$. The second application is identical except that all  subscripts  ``0'' are substituted with ``u''. For lack of space we do not derive the exact bounds.

The distortion error occurs because each hash function that is utilized by user $u$ can self-collide:
\begin{equation}
 \epsilon_d=\!\sum_{v\in\{u,0\}}\! |\inner{\phi_v(x)}{\phi_v(w_v)}-\inner{x}{w_v}|.
\end{equation}
To show that $\epsilon_d$ is small with high probability, we apply Corollary~\ref{cor:inner} once for each possible values of $v$.

In section~\ref{sec:results} we show experimental results for this
setting.  The empirical results are stronger than the theoretical
bounds derived in this subsection---our technique outperforms a single
global classifier on hundreds thousands of users.  We discuss
an intuitive explanation in section~\ref{sec:results}.

\paragraph{Massively Multiclass Estimation}
We can also regard massively multi-class classification as a multitask problem, and apply feature hashing in a way similar to the personalization setting. Instead of using a different hash function for each user, we use a different hash function for each class.  

\cite{Shietal09} apply feature hashing to problems with a high number
of categories. They show empirically that \emph{joint} hashing of the
feature vector $\phi(x,y)$ can be efficiently achieved for problems
with millions of features and thousands of classes.
%

\paragraph{Collaborative Filtering}

Assume that we are given a very large sparse matrix $M$ where the entry $M_{ij}$ indicates what action user $i$ took on instance $j$. A common example for actions and instances is
user-ratings of movies~\cite{netflix}. A successful method for finding common factors amongst users and instances for predicting unobserved actions is to factorize
$M$ into $M=U^\top W$. If we have millions of users performing millions of actions, storing $U$ and $W$ in memory quickly becomes infeasible.
Instead, we may choose to compress the matrices $U$ and $W$ using hashing. For $U, W \in \RR^{n \times d}$ denote by $u, w \in \RR^\N$ vectors with
\begin{align}
  \nonumber
  u_{i} = \hspace{-2ex} \sum_{j,k : h(j,k) = i} \hspace{-2ex} \xi(j,k) U_{jk}
  \text{ and }
  w_{i} = \hspace{-2ex} \sum_{j,k : h'(j,k) = i} \hspace{-2ex} \xi'(j,k) W_{jk}.
\end{align}
where $(h,\xi)$ and $(h',\xi')$ are independently chosen hash
functions.  This allows us to approximate matrix elements $M_{ij} =
[U^\top W]_{ij}$ via
\begin{align*}
  M^\phi_{ij} := \sum_k \xi(k,i) \xi'(k,j) u_{h(k,i)} w_{h'(k,j)}.
\end{align*}
This gives a compressed vector representation of $M$ that can be efficiently stored.

\section{Results}
\label{sec:results}

We evaluated our algorithm in the setting of personalization. As data set, we used a proprietary email spam-classification task of $n=3.2$ million emails, properly anonymized, collected from $|U|=433167$ users. Each email is labeled as \emph{spam} or \emph{not-spam} by one user in $U$. After tokenization, the data set consists of $40$ million unique words.

\begin{figure}[t!]
  \centering
  \includegraphics[width=0.5\textwidth]{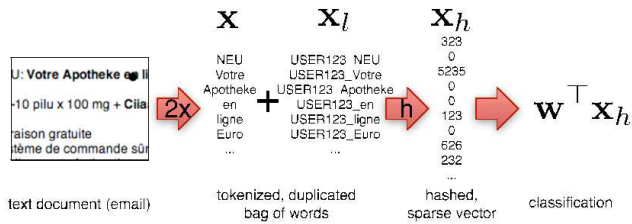}
  \caption{The hashed personalization summarized in a schematic
    layout. Each token is duplicated and one copy is individualized (e.g. by concatenating each word with a unique user identifier). Then, the global hash function maps all tokens  into a low dimensional feature space where the document is classified.}
  \label{fig:schema}
\end{figure}

For all experiments in this paper, we used the Vowpal Wabbit implementation\footnote{{http://hunch.net/$\sim$vw/}} of stochastic gradient descent on a square-loss.
In the mail-spam literature the misclassification of \emph{not-spam} is considered to be much more harmful than misclassification of \emph{spam}. 
We therefore follow the convention to set the classification threshold during test time such that exactly $1\%$ of the $not-spam$ test data is classified as $spam$
%
Our implementation of the  personalized hash functions is illustrated in Figure~\ref{fig:schema}. To obtain a personalized hash function $\phi_u$ for user $u$, we concatenate a unique user-id to each word in the email and then hash the newly generated tokens with the same global hash function.
%

The data set was collected over a span of 14 days. We used the first 10 days for training and the remaining 4 days for testing. As \emph{baseline}, we chose the purely global classifier trained over all users and hashed into $2^{26}$ dimensional space. As $2^{26}$ far exceeds the total number of unique words we can regard the baseline to be representative for the classification without hashing. All results are reported as the amount of spam that passed the filter undetected, relative to this baseline (eg. a value of $0.80$ indicates a $20\%$ reduction in spam for the user)\footnote{As part of our data sharing agreement, we agreed not to include absolute classification error-rates.}.
%

Figure~\ref{fig:global_results} displays the average amount of spam in users' inboxes as a function of the number of hash keys $m$, relative to the baseline above. In addition to the baseline, we evaluate two different settings.

\begin{figure}[t]
  \centering
  \includegraphics[width=\columnwidth]{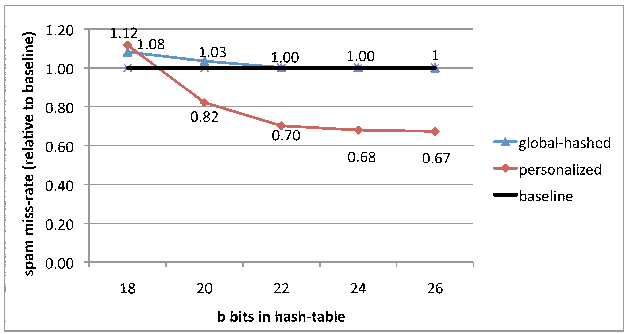}
  \caption{The decrease of uncaught spam over the baseline classifier averaged over all users. The classification threshold was chosen to keep the not-spam misclassification fixed at $1\%$. The hashed global classifier (\emph{global-hashed}) converges relatively soon, showing that the distortion error $\epsilon_d$ vanishes. The personalized classifier results in an average improvement of up to $30\%$. }
  \label{fig:global_results}
\end{figure}
The \emph{global-hashed} curve represents the relative spam catch-rate of the global classifier after hashing $\inner{\phi_0(w_0)}{\phi_0(x)}$. At $m=2^{26}$ this is identical to the baseline. Early convergence at $m=2^{22}$ suggests that at this point hash collisions have no impact on the classification error and the {\em baseline} is indeed equivalent to that obtainable without hashing.

In the \emph{personalized} setting each user $u\in U$ gets her own
classifier $\phi_u(w_u)$ as well as the global classifier
$\phi_0(w_0)$.  Without hashing the feature space explodes, as the
cross product of $u=400K$ users and $n=40M$ tokens results in $16$
trillion possible unique personalized features. Figure~\ref{fig:global_results}
shows that despite aggressive hashing, personalization results in a
$30\%$ spam reduction once the hash table is indexed by $22$ bits.

\paragraph{User clustering}
One hypothesis for the strong results in Figure~\ref{fig:global_results} might originate from the non-uniform distribution of user votes --- it is possible that using personalization and feature hashing we benefit a small number of users who have labeled many emails, degrading the performance of most users (who have labeled few or no emails) in the process. In fact, in real life, a large fraction of email users do not
contribute at all to the training corpus and only interact with the
classifier during test time. The personalized version of the test
email $\Phi_u(x_u)$ is then hashed into buckets of other tokens and only
adds interference noise $\epsilon_i$ to the classification.

\begin{figure}[t!]
  \centering
  \includegraphics[width=\columnwidth]{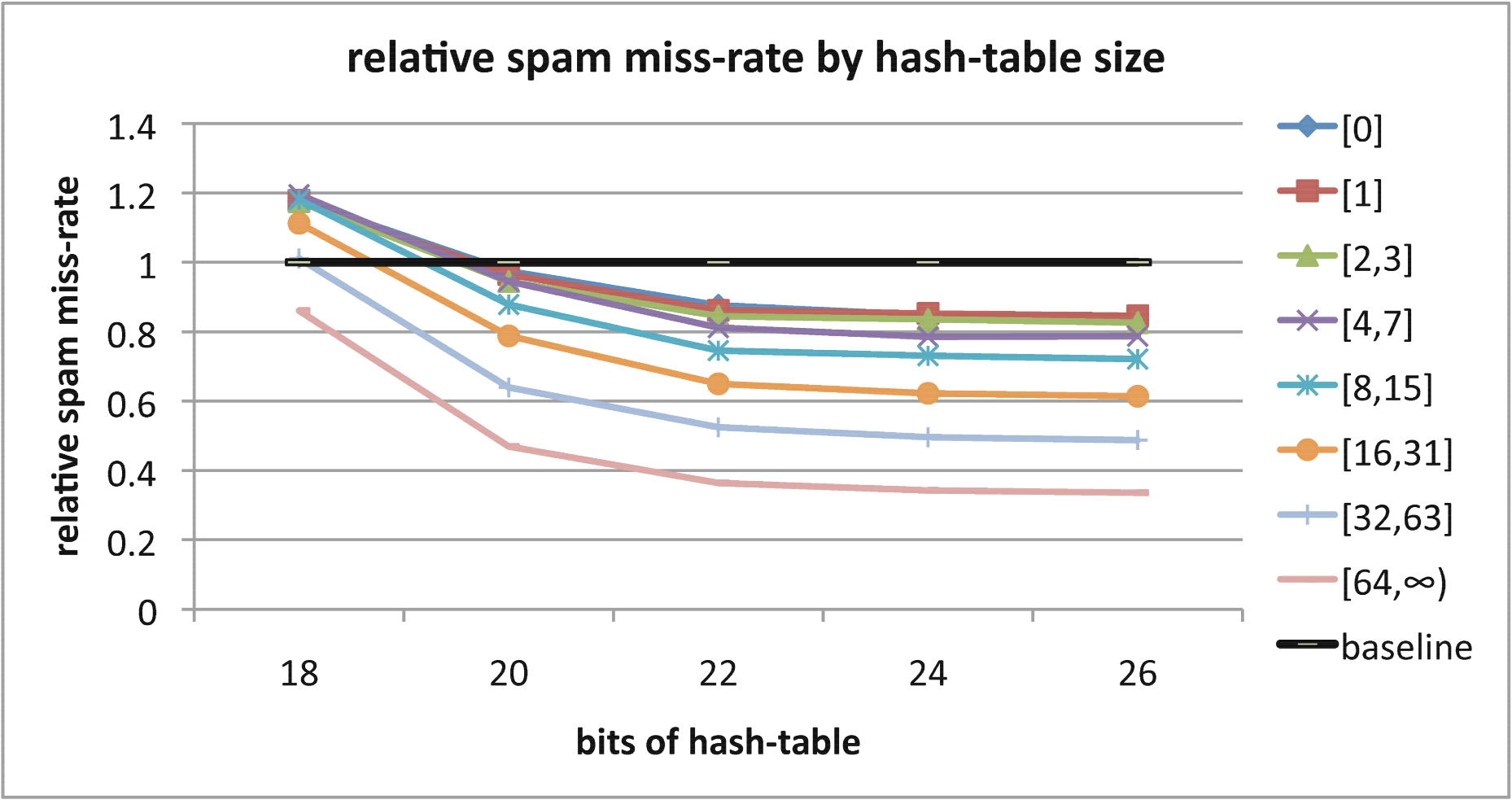}
  \caption{Results for users clustered by training emails. For
    example, the bucket $[8,15]$ consists of all users with eight to
    fifteen training emails. Although users in buckets with large
    amounts of training data do benefit more from the personalized
    classifier (up-to $65\%$ reduction in spam), even users that did
    not contribute to the training corpus at all obtain almost $20\%$
    spam-reduction.}
  \label{fig:figures_results}
\end{figure}
In order to show that we improve the performance of most users, it is therefore important that we not
 only report averaged results over all emails, but explicitly examine the effects of the
personalized classifier for users depending on their contribution to
the training set. To this end, we place users into exponentially growing buckets
based on their number of training emails and compute the relative
reduction of uncaught spam for each bucket
individually. Figure~\ref{fig:figures_results} shows the results on a
per-bucket basis.  We do not compare against a \emph{purely} local
approach, with no global component, since for a large fraction of
users---those without training data---this approach cannot outperform
random guessing.

It might appear rather surprising that users in the bucket with none
or very little training emails (the line of bucket $[0]$ is identical
to bucket $[1]$) also benefit from personalization. After all, their
personalized classifier was never trained and can only add noise at
test-time. The classifier improvement of this bucket can be explained
by the subjective definition of \emph{spam} and \emph{not-spam}. In
the personalized setting the individual component of user labeling is
absorbed by the local classifiers and the global classifier represents
the \emph{common} definition of spam and not-spam. In other words, the
global part of the personalized classifier obtains better
generalization properties, benefiting all users.

\section{Related Work}

A number of researchers have tackled related, albeit different
problems.

{\bfseries \cite{RahRec08}} use Bochner's theorem and
sampling to obtain approximate inner products for Radial Basis
Function kernels. \cite{RahRec09} extend this to sparse
approximation of weighted combinations of basis
functions. This is computationally efficient for many function
spaces. Note that the representation is
\emph{dense}.

{\bfseries \cite{LiChuHas07}} take a complementary approach: for
sparse feature vectors, $\phi(x)$, they devise a scheme of reducing the
number of nonzero terms even further. While this is in principle
desirable, it does not resolve the problem of $\phi(x)$ being high
dimensional. More succinctly, it is necessary to express the function in the
dual representation rather than expressing $f$ as a linear function,
where $w$ is unlikely to be compactly represented: $f(x) = \inner{\phi(x)}{w}$.

{\bfseries \cite{Achlioptas03}} provides computationally efficient
randomization schemes for dimensionality reduction. Instead of
performing a dense $d \cdot \N$ dimensional matrix vector
multiplication to reduce the dimensionality for a vector of
dimensionality $d$ to one of dimensionality $\N$, as is required by
the algorithm of \cite{GioIndMot99}, he only requires $\frac{1}{3}$ of
that computation by designing a matrix consisting only of entries
$\cbr{-1, 0, 1}$. Pioneered by {\bfseries \cite{ailon2006ann}}, there
has been a line of work~\cite{ailon2008fdr,matousek2008vjl} on
improving the complexity of random projection by using various
code-matrices in order to preprocess the input vectors. Some of our
theoretical bounds are derivable from that of {\bfseries
  ~\cite{liberty2008dfr}}.

A related construction is the CountMin sketch of
{\bfseries \cite{CorMut04}} which stores counts in a number of replicates of a
hash table. This leads to good concentration inequalities for range
and point queries.

{\bfseries \cite{Shietal09}} propose a hash kernel to deal with the
issue of computational efficiency by a very simple algorithm:
high-dimensional vectors are compressed by adding up all coordinates
which have the same hash value --- one only needs to perform as many
calculations as there are nonzero terms in the vector. This is a
significant computational saving over locality sensitive hashing
\cite{Achlioptas03,GioIndMot99}.

Several additional works provide motivation for the investigation of
hashing representations.  For example, {\bfseries
  \cite{GanchevDredze08}} provide empirical evidence that the hashing
trick can be used to effectively reduce the memory footprint on
many sparse learning problems by an order of magnitude via removal of
the dictionary.  Our experimental results validate this, and show that
much more radical compression levels are achievable.  In addition,
{\bfseries \cite{LaLiSt07}} released the Vowpal Wabbit fast online
learning software which uses a hash representation similar to that
discussed here.


\section{Conclusion}
\label{sec:conclusion}
In this paper we analyze the hashing-trick for dimensionality
reduction  theoretically and  empirically. As part of our theoretical analysis we introduce unbiased hash functions and  provide exponential tail bounds for hash kernels. These give further inside into hash-spaces and explain previously made empirical observations. We also derive that random subspaces of the hashed space are likely to not interact, which makes multitask learning with many tasks possible.

Our empirical results validate this on a real-world application within
the context of spam filtering. Here we demonstrate that
even with a very large number of tasks and features, all mapped into
a joint lower dimensional hash-space, one can obtain impressive classification results with finite memory guarantee.

\bibliography{bibfile}
\bibliographystyle{mlapa}
\begin{appendix}

\section{Mean and Variance}

\label{proof:mean}
\begin{proof} [Lemma~\ref{th:meanvar}]
  To compute the expectation we expand
\begin{equation}
    \inner{x}{x'}_\phi = \sum_{i,j} \xi(i) \xi(j) x_i x_j' \delta_{h(i), h(j)}.	
\end{equation}
Since $\Eb_\phi[\inner{x}{x'}_\phi] = \Eb_h[ \Eb_\xi[\inner{x}{x'}_\phi ]]$, taking expectations over $\xi$ we see that only the terms $i = j$ have
  nonzero value, which shows the first claim. For the variance we compute $\Eb_\phi[\inner{x}{x'}_\phi^2]$.  Expanding this, we get:
\begin{equation*}
	    \inner{x}{x'}^2_\phi = \sum_{i,j,k,l} \xi(i) \xi(j) \xi(k) \xi(l) x_i {x}_j' x_k {x}_l' \delta_{h(i), h(j)}	\delta_{h(k), h(l)}.
\end{equation*}
This expression can be simplified by noting that:
  \begin{align*}
    \Eb_\xi \sbr{\xi(i) \xi(j) \xi(k) \xi(l) } =
    \delta_{ij} \delta_{kl} + [1 - \delta_{ijkl}] ( \delta_{ik} \delta_{jl} + \delta_{il}\delta_{jk}).
  \end{align*}
Passing the expectation over $\xi$ through the sum, this allows us to
break down the expansion of the variance into two terms.
  \begin{align*}
    & \Eb_{\phi}[\inner{x}{x'}_\phi^2] = \sum_{i,k} x_i x_i' x_k x_k'
    + \sum_{i \neq j} x_i^2 {x_j'}^2 \Eb_h\sbr{\delta_{h(i),h(j)}} \\& + \sum_{i \neq j} x_i x_i' x_j x_j' \Eb_h\sbr{\delta_{h(i),h(j)}}\\
    & = \inner{x}{x'}^2 + \frac{1}{\N}\left( \sum_{i \neq j} x_i^2 {x_j'}^2 + \sum_{i \neq j} x_i x_i' x_j x_j'\right)
  \end{align*}
  by noting that $\Eb_h\sbr{\delta_{h(i),h(j)}} = \frac{1}{\N}$ for $i\neq j$. Using the fact that $\sigma^2 = \Eb_\phi[\inner{x}{x'}_\phi^2] -
  \Eb_\phi[\inner{x}{x'}_\phi]^2$ proves the claim.
\end{proof}

\section{Concentration of Measure}
\label{sec:proofs}
We use the concentration result derived by Liberty, Ailon and Singer in~\cite{liberty2008dfr}. Liberty et al. create a Johnson-Lindenstrauss random projection matrix by combining a carefully constructed deterministic matrix $A$ with random diagonal matrices. For completeness we restate the relevant lemma. Let $i$ range over the hash-buckets. Let $m = c\log(1/\delta)/\epsilon^2$ for a large enough constant $c$. For a given vector $x$, define the diagonal matrix $D_x$ as $(D_x)_{jj} = x_j$. For any matrix $A\in \Re^{m\times d}$, define $\|x\|_A \equiv \max_{y:\|y\|_2=1}\|AD_xy\|_2$.
\paragraph{Lemma 2~\cite{liberty2008dfr}.} {\em For any column-normalized matrix $A$, vector $x$ with $\|x\|_2=1$ and an i.i.d. random $\pm 1$ diagonal matrix  $D_s$, the following holds:} $\forall x, \mbox{ if }\|x\|_A \leq \frac{\epsilon}{6\sqrt{\log(1/\delta)}}\mbox{ then, } \Pr[| \|AD_s x\|_2 - 1 | > \epsilon] \leq \delta.$

\par We also need the following form of a weighted balls and bins inequality -- the statement of the Lemma, as well as the proof follows that of Lemma 6~\cite{dsr2010sparse}. We still outline the proof because of some parameter values being different.
\begin{lemma}
\label{lem:ballsbins}
Let $m$ be the size of the hash function range and let $\eta = \frac{1}{2\sqrt{m\log(m/\delta)}}$. If $x$ is such that $\|x\|_2=1$ and $\|x\|_\infty \leq \eta$, then define $\sigma^2_* = \max_i \sum_{j=1}^d x_j^2 \delta_{ih(j)}$ where $i$ ranges over
all hash-buckets. We have that with probability $1 - \delta$,
\begin{align*}
\sigma_*^2 \leq \frac{2}{m}
\end{align*}
\end{lemma}
\begin{proof}
We outline the proof-steps. Since the buckets have identical distribution, we look only at the $1^{st}$ bucket, i.e. at $i=1$ and bound $\sum_{j:h(j)=1} x_j^2$. Define $X_j = x_j^2 \left( \delta_{1h(j)} - \frac{1}{m} \right)$. Then $E_h[X_j] = 0$ and $E_h[X_j^2 ] = x_j^4 \left( \frac{1}{m} - \frac{1}{m^2}\right) \leq \frac{x_j^4}{m} \leq \frac{x_j^2\eta^2}{m}$ using $\|x\|_\infty \leq \eta$. Thus, $\sum_{j} E_h[X_j^2]  \leq \frac{\eta^2}{m}$. Also note that $\sum_j X_j = \sum_{j:h(j)=1} x_j^2 - \frac{1}{m}$.
Plugging this into the Bernstein's inequality, equation~\ref{def:Bernstein}, we have that
\begin{align*}
& \Pr[ \sum_j X_j > \frac{1}{m} ] \leq \exp\left( -\frac{1/2m^2}{\eta^2/m + \eta^2/3m} \right)\\
& =  \exp( - \frac{3}{8m\eta^2} ) \leq \exp( -\log(m/\delta)) \leq \delta/m
\end{align*}
By taking union bound over all the $m$ buckets, we get the above result.
\end{proof}

\begin{proof}[Theorem~\ref{th:key}] Given the function $\phi=(h, r)$, define the matrix $A$ as $A_{ij} = \delta_{ih(j)}$ and $D_s$ as $(D_s)_{jj} = r_j$. Let $x$ be as specified, i.e. $\|x\|_2=1$ and $\|x\|_\infty \leq \eta$. Note that $\|x\|_\phi = \|AD_s x\|_2$.
Let $y\in \Re^d$ be such that $\|y\|_2 = 1 $. Thus
\begin{align*}
\|AD_x y\|_2^2  & = \sum_{i=1}^m \left( \sum_{j=1}^d y_j \delta_{ih(j)} x_j\right)^2 \\
& \leq \sum_{i=1}^m (\sum_{j=1}^d y_j^2 \delta_{ih(j)}) (\sum_{j=1}^d x_j^2 \delta_{ih(j)}) \\
& \leq  \sum_{i=1}^m (\sum_{j=1}^d y_j^2 \delta_{ih(j)}) \sigma_{*}^2 \leq \sigma_*^2.
\end{align*}
by applying the Cauchy-Schwartz inequality, and using the definition of $\sigma_*$. Thus, $\|x\|_A = \max_{y:\|y\|_2=1}\|AD_xy\|_2 \leq \sigma_* \leq \sqrt{2}m^{-1/2}$.
If $m \geq \frac{72}{\epsilon^2}\log(1/\delta)$, we have that $\|x\|_A  \leq \frac{\epsilon}{6\sqrt{\log(1/\delta)}}$, which satisfies the conditions of Lemma 2 from~\cite{liberty2008dfr}. Thus applying the above result from Lemma 2~\cite{liberty2008dfr} to $x$, and using Lemma~\ref{lem:ballsbins}, we have that
$\Pr [ | \|AD_sx\|^2- 1 | \geq \epsilon ] \leq \delta$
and hence
\begin{align*}
\Pr [ | \|x\|^2_\phi - 1 | \geq \epsilon ] \leq \delta
\end{align*}
by taking union over the two error probabilities of Lemma 2 and Lemma~\ref{lem:ballsbins}, we have the result.
\end{proof}

\section{Inner Product}
\label{sec:innerproof}
\begin{proof}[Corollary~\ref{cor:inner}] We have that $2\inner{x}{x'}_\phi = \nbr{x}_\phi^2 +  \nbr{x'}_\phi^2 -  \nbr{x - x'}_\phi^2$.
Taking expectations, we have the standard inner product inequality. Thus,
\begin{align*}
& | 2\inner{x}{x'}_\phi - 2\inner{x}{x'} |  \leq | \nbr{x}_\phi^2 - \nbr{x}^2 | \\
& + | \nbr{x'}_\phi^2 - \nbr{x'}^2| + | \nbr{x - x'}_\phi^2 - \nbr{x - x'}^2 |
\end{align*}
Using union bound, with probability $ 1 - 3\delta$, each of the terms above is bounded using Theorem~\ref{th:key}. Thus, putting the bounds together, we have that, with probability $ 1 - 3\delta$,
\begin{align*}
& | 2\inner{\phi_u(x)}{\phi_u(x)} - 2\inner{x}{x} | \leq    \epsilon ( \nbr{x}^2 + \nbr{x'}^2 + \nbr{x - x'}^2 )
\end{align*}

\end{proof}
\section{Refutation of the Previous Incorrect Proof}

There were a few bugs in the previous version of the paper~\cite{weinberger2009fhl}. We now detail each of them and illustrate why it was an error. The current result shows that the using hashing we can create a projection matrix that can preserve distances to a factor of $(1\pm \epsilon)$ for vectors with a bounded $\|x\|_\infty/\|x\|_2$ ratio. The constraint on input vectors can be circumvented by multiple hashing, as outlined in Section~\ref{subsec:multiplehashing}, but that would require hashing $O(\frac{1}{\epsilon^2})$ times. Recent work~\cite{dsr2010sparse} suggests that better theoretical bounds can be shown for this construction.
We thank Tamas Sarlos and Ravi Kumar for the following writeup on the errors and for suggestion the new proof in Appendix~\ref{sec:proofs}.

\begin{enumerate}
\item The statement of the main theorem in Weinberger
et al.~\cite[Theorem 3]{weinberger2009fhl} is false as it contradicts the
lower bound of Alon~\cite{alon2003par}. The flaw lies in the probability of error in
\cite[Theorem~3]{weinberger2009fhl}, which
was claimed to be $\exp(-\frac{\sqrt{\epsilon}}{4\eta})$.  This error
can be made arbitrarily small without increasing the embedding dimensionality
$m$ but by decreasing $\eta=\frac {||x||_{\infty}}{||x||_2}$, which in turn
can be achieved by preprocessing the input vectors $x$.  However, this
contradicts Alon's lower bound on the embedding dimensionality. The details
of this contradiction are best presented through
\cite[Corollary 5]{weinberger2009fhl} as follows.

Set $m=128$ and $\delta=1/2$ and consider the vertices of the
$n$-simplex in $\Re^{n+1}$, i.e., $x_1=(1,0,...,0)$, $x_2=(0,1,0,...,0)$,
\ldots.
Let $P \in \Re^{(n+1)c \times (n+1)}$ be the naive, replication based
preconditioner, with replication parameter $c=512\log^2 n$ as defined in
Section~2 of our submission or \cite[Section 3.2]{weinberger2009fhl}.
Therefore for all pairs $i \ne j$ we have that
$||Px_i-Px_j||_{\infty}=1/\sqrt{c}$ and that $||Px_i-Px_j||_2=\sqrt{2}$.
Hence we can apply \cite[Corollary 5]{weinberger2009fhl} to the set of
vectors $Px_i$ with $\eta=1/\sqrt{2c}=1/(32\log n)$;
then the claimed approximation error is
$\sqrt{\frac{2}{m}}+64\eta^2\log^2\frac{n}{2\delta} =
\frac 1 8 + \frac 1 {16} \le \frac 1 4$.
If Corollary~5 were true, then it would follow that
with probability at least $1/2$, the linear
transformation $A=\phi\cdot P : \Re^{n+1}\rightarrow \Re^m$
distorts the pairwise distances of the above $n+1$ vectors by at most a
$1\pm1/4$ multiplicative factor.
On the other hand, the lower bound of Alon shows that any such transformation
$A$ must map to $\Omega(\log n)$ dimensions;
see the remarks following Theorem 9.3 in \cite{alon2003par} and
set $\epsilon=1/4$ there. This clearly contradicts $m=128$ above.

\item The proof of the Theorem 3 contained a fatal, unfixable error.
Recall that $\delta_{ij}$ denotes the usual Kronecker symbol, and $h$ and $h'$
are hash functions. 
Weinberger et al.~make the following observation after equation~(13) of
their proof on page 8 in Appendix B.
\begin{quote}
``First note that $\sum_i\sum_j\delta_{h(j)i}+\delta_{h'(j)i}$ is at most
$2t$, where $t=|\{j: h(j) \ne h'(j) \}|$.''
\end{quote}
The quoted observation is false. Let $d$ denote the dimension of the input.
Then, $\sum_i\sum_j\delta_{h(j)i}+\delta_{h'(j)i}=
\sum_j(\sum_i\delta_{h(j)i}+\delta_{h'(j)i})=\sum_j 2=2d$,
independent of the choice of the hash  function. Note that $t$ played a
crucial role in the proof of~\cite{weinberger2009fhl} relating the Euclidean
approximation error of the dimensionality reduction to Talagrand's convex
distance defined over the set of hash functions. Albeit the error is
elementary, we do not see
how to rectify its consequences in~\cite{weinberger2009fhl}
even if the claim were of the right form.

\item The proof of Theorem 3 in~\cite{weinberger2009fhl} also contains a minor and fixable error.  To see this, consider the
sentence towards the end of the
proof Theorem~3 in~\cite{weinberger2009fhl} where $0<\epsilon<1$ and
$\beta = \beta(x) \ge 1$.
\begin{quote}
``Noting that $s^2=(\sqrt{\beta^2+\epsilon}-\beta)/4||x||_{\infty}\ge
\sqrt{\epsilon}/4||x||_{\infty}$, ...''
\end{quote}
Here the authors wrongly assume that $\sqrt{\beta^2+\epsilon}-\beta \ge
\sqrt{\epsilon}$ holds, whereas the truth is
$\sqrt{\beta^2+\epsilon}-\beta \le \sqrt{\epsilon}$ always.

Observe that this glitch is easy to fix locally, however this
change is minor and the modified claim would still be false. Since for all
$0\le y\le 1$ we have that $\sqrt{1+y}\ge 1+y/3$, from $\beta \ge 1$
it follows that $\sqrt{\beta^2+\epsilon} -\beta \ge \epsilon/3$. Plugging
the latter estimate into the ``proof'' of Theorem~3 would result in a modified
claim where the original probability of error,
$\exp(-\frac{\sqrt{\epsilon}}{4\eta})$, is replaced with
$\exp(-\frac{\epsilon}{12\eta})$.  Updating the numeric constants
in the first section of this note would show that the new claim still
contradicts Alon's lower bound. To justify observe that counter example is
based on a constant $\epsilon$ and the modified claim would still lack
the necessary $\Omega(\log n)$ dependency in its target dimensionality.
\end{enumerate}

\end{appendix}

\end{document}